\documentclass[11pt]{article}

\usepackage{psfrag,amsmath,amsbsy,amsfonts,amssymb,amsthm,fullpage}
\usepackage{graphicx}
\usepackage{algorithm,algorithmic,mathtools}
\usepackage{natbib}
\usepackage{color}
\usepackage{tikz}
\usetikzlibrary{calc,shapes}
\usepackage{subfig}
\usepackage{hyperref}
\usepackage{multirow}

\theoremstyle{plain}

\theoremstyle{definition}

\theoremstyle{remark}

\newcommand{\vcd}{\mathrm{VCD}}
\newcommand{\rtd}{\mathrm{RTD}}
\newcommand{\td}{{\mathrm{TD}}}
\newcommand{\tdmin}{\mathrm{TD}_{\min}}
\newcommand{\calC}{{\mathcal C}}

\newtheorem{theorem}{Theorem}
\newtheorem{lemma}[theorem]{Lemma}

\newtheorem{definition}[theorem]{Definition}

\title{Quadratic Upper Bound for Recursive Teaching Dimension of Finite VC Classes}

\author{Lunjia Hu, Ruihan Wu, Tianhong Li and Liwei Wang
\footnote{Lunjia Hu, Ruihan Wu and Tianhong Li are with Institute for Interdisciplinary Information Sciences, Tsinghua University, China. Email: hulj14@mails.tsinghua.edu.cn, wrh14@mails.tsinghua.edu.cn, lth14@mails.tsinghua.edu.cn. Liwei Wang is with Key Laboratory of Machine Perception, School of Electronics Engineering and Computer Sciences, Peking University, China. Email: wanglw@cis.pku.edu.cn.}}

\begin{document}
\maketitle

\begin{abstract}
In this work we study the quantitative relation between the recursive teaching dimension (RTD) and the VC dimension (VCD) of concept classes of finite sizes. The RTD of a concept class $\calC \subseteq \{0, 1\}^n$, introduced by \citet{ZLHZ2011}, is a combinatorial complexity measure characterized by the worst-case number of examples necessary to identify a concept in $\calC$ according to the recursive teaching model.

For any finite concept class $\calC \subseteq \{0,1\}^n$ with $\vcd(\calC)=d$, \citet{SZ2015} posed an open problem $\rtd(\calC) = O(d)$, i.e., is RTD linearly upper bounded by VCD? Previously, the best known result is an exponential upper bound $\rtd(\calC) = O(d \cdot 2^d)$, due to \citet{CCT2016}. In this paper, we show a quadratic upper bound: $\rtd(\calC) = O(d^2)$, much closer to an answer to the open problem. We also discuss the challenges in fully solving the problem.
\end{abstract}



\section{Introduction}\label{Section:Introduction}

Sample complexity is one of the most important concepts in machine learning. Basically, it is the amount of data needed to achieve a desired learning accuracy. Sample complexity has been extensively studied in various learning models. In PAC-learning, sample complexity is characterized by the VC dimension (VCD) of the concept class \citep{BEHW1989,VC1971}. PAC-learning is a passive learning model. In this model, the role of the teacher is limited to providing labels to data \emph{randomly} drawn from the underlying distribution.

Different from PAC-learning, there are important models in which teacher involves more actively in the learning process. For example, in the classical teaching model \citep{GK1995,SM1991}, the teacher chooses a set of labeled examples so that the learner, after receiving the examples, can distinguish the target concept from all other concepts in the concept class. In this model, the key complexity measure of a concept class is the teaching dimension, which is defined as the worst-case number of examples needed to be selected by the teacher \citep{GK1995}. Teaching dimension finds applications in many learning problems \citep{A2004,H2007,D2005, H1995, GM1993, ABS1995}.

Another model of teaching, the recursive teaching model, is proposed by \citet{ZLHZ2011}. The idea underlying the recursive teaching model is to let the teacher exploit a hierarchical structure in the concept class. Concretely, the hierarchy of a concept class is a nesting, starting from the concept that requires the smallest amount of data to teach, and then applying this process recursively to the rest of the concepts. The complexity measure of a concept class in the recursive teaching model is called recursive teaching dimension (RTD). RTD is defined as the worst-case number of examples needed to be selected by the teacher for any target concept during the recursive process \citep{ZLHZ2011}. See also Section \ref{sec:preli} for a formal definition.

Although less intuitive, RTD exhibits surprising properties. The most interesting property is the quantitative relation between RTD and VCD of a finite concept class. As an example, for any finite \emph{maximal} class $\calC \subseteq \{0,1\}^n$, i.e., the size of $\calC$ equals the Sauer bound, it can be shown that $\rtd(\calC) = \vcd(\calC)$  \citep{DFSZ2014}. The importance of this result is that maximal classes contain many natural classes such as the arrangement of half spaces. Another special case is the intersection-closed concept classes. For such a concept class $\calC$, $\rtd(\calC) \le \vcd(\calC)$. On the other hand, there exist cases where $\rtd(\calC) > \vcd(\calC)$. However, the best known worst-case lower bound is $\rtd(\calC) \ge \frac{5}{3} \vcd(\calC)$, which is proven by giving an explicit construction \citep{CCT2016}. For more special cases such that $\rtd(\calC) = \vcd(\calC)$ or $\rtd(\calC) \le \vcd(\calC)$, please refer to \citep{DFSZ2014}.

Based on these insights, \citet{SZ2015} posed an open problem on the quantitative relation between RTD and VCD of general concept classes: For any finite concept class $\calC \subseteq \{0,1\}^n$ with $\vcd(\calC)=d$, is $\rtd(\calC)$ linearly upper bounded by $d$, i.e., does $\rtd(\calC) \le \kappa d$ hold for a universal constant $\kappa$?

At the time when this open problem was posed, the only known result for general concept classes $\calC \subseteq \{0,1\}^n$ is $\rtd(\calC) = O(d\cdot 2^d \log\log|\calC|)$ \citep{MSWY2015}. This bound is exponential in VCD and depends on the size of the concept class.

Before our work, the best known upper bound is due to \citet{CCT2016}, who proved that $\rtd(\calC) = O(d\cdot 2^d)$, which is the first upper bound for $\rtd(\calC)$ that depends only on $\vcd(\calC)$, but not on the size of the concept class.

In this paper, we continue this line of research and extend the techniques developed in \citep{K1999,MSWY2015,CCT2016}. Our main result is a quadratic upper bound $\rtd(\calC) = O(d^2)$ for any finite concept class $\calC \subseteq \{0,1\}^n$ with $\vcd(\calC)=d$. In particular, we prove $\rtd(\calC) \le 39.3752 d^2 - 3.6330 d$. Comparing to previous results, our bound is much closer to the linear upper bound in the open problem.

As pointed out by \citet{SZ2015}, a solution to their open problem will have important implications: It provides deeper understanding not only of the relationship between the complexity of teaching and the complexity of passive supervised learning, but also on the well-known sample compression conjecture \citep{W2003,LW1986}, which states that for every concept class of VCD $d$, there is a compression scheme that can compress the samples to a subset of size at most $d$. (See also \citep{DMY2016} for recent progress.)

The rest of the paper is organized as follows. Section \ref{sec:preli} presents the background and all the definitions. In Section \ref{sec:main} we propose our main results and proofs. Section \ref{sec:discussion} provides discussions on the challenges in fully solving the open problem.

\section{Preliminaries}
\label{sec:preli}

Let $X$ be a finite instance space and $\calC$ a concept class over $X$, i.e., $\calC \subseteq \{0,1\}^X$. For notational simplicity, we always assume $X = [n]$ where $[n]= \{1,2,\ldots,n\}$, and consider concept class $\calC \subseteq \{0,1\}^n$.

The VC dimension of a concept class $\calC \subseteq \{0,1\}^n$, denoted by $\vcd(\calC)$, is the maximum size of a shattered subset of $[n]$, where $A \subseteq [n]$ is said to be shattered by $\calC$ if $|\{c|_{A}:~ c \in \calC\}| = 2^{|A|}$. Here $c|_{A}$ is the projection of $c$ on $A$. In other words, for every $b \in \{0,1\}^{|A|}$, there is $c \in \calC$ so that $c|_{A}=b$.

For a given concept class $\calC \subseteq \{0,1\}^n$ and a concept $c \in \calC$, we say $A \subseteq [n]$ is a teaching set for $c$ if $A$ distinguishes $c$ from all other concepts in $\calC$. That is, $c|_{A} \neq c'|_{A}$ for all $c' \in \calC$, $c' \neq c$.

The size of the \emph{smallest} teaching set for $c$ with respect to $\calC$ is denoted by $\td(c;\calC)$. In the classical teaching model \citep{GK1995,SM1991}, the teaching dimension of a concept class $\calC$, denoted by $\td(\calC)$, is defined as $\td(\calC) = \max_{c \in \calC}\td(c; \calC)$. $\td(\calC)$ can be seen as the worst-case teaching complexity \citep{K1999}, as it considers the hardest concept to distinguish from other concepts. However, defining teaching complexity using the hardest concept is often restrictive; and $\td(\calC)$ does not always capture the idea of cooperation in teaching and learning. In fact, a simple concept class may have the maximum possible complexity \citep{ZLHZ2011}. Instead, one can consider the best-case teaching dimension of $\calC$.

\begin{definition}[Best-Case Teaching Dimension]
The best-case teaching dimension of a concept class $\calC$, denoted by $\tdmin(\calC)$, is defined as
\[
\tdmin(\calC) = \min_{c \in \calC} \td(c; \calC).
\]
\end{definition}

In the recursive teaching model \citep{ZLHZ2011}, the teacher exploits a hierarchy of the concept class $\calC$. It recursively removes from the given concept class all concepts whose teaching dimension with respect to the remaining concepts is smallest. The recursive teaching dimension $\rtd$ of $\calC$ is defined as the largest value of the smallest teaching dimensions encountered in the recursive process.

\begin{definition}[Recursive Teaching Dimension \citep{ZLHZ2011}]
For a given concept class $\calC$, define a sequence $\calC_0,\calC_1,\ldots,\calC_T$ such that $\calC_0 = \calC$, and $\calC_{t+1} = \calC_t \backslash \{c \in \calC_t:\td(c;\calC_t)=\tdmin(\calC_t)\}$. Here $T$ is the smallest integer so that $\calC_{T+1} = \emptyset$. The recursive teaching dimension of $\calC$, denoted by $\rtd(\calC)$, is defined as $\rtd(\calC) = \max_{0 \le t \le T}\tdmin(\calC_t)$.
\end{definition}

Our goal is to bound $\rtd(\calC)$ in terms of $\vcd(\calC)$. It turns out that rather than studying VC dimension and shattering directly, considering the number of projection patterns is more helpful \citep{K1999,MSWY2015}.

\begin{definition}[$(x,y)$-class]
We say a concept class $\calC \subseteq \{0,1\}^n$ is an $(x,y)$-class for positive integers $x,y$, if for any $A \subseteq [n]$ such that $|A| \le x$, $|\{c|_{A}:c \in \calC\}| \le y$.
\end{definition}

In the rest of this paper we will frequently use the following observations. A concept class $\calC$ with $\vcd(\calC)=d$ is $(x,2^x)$-class for every $x \le d$. More importantly, $\calC$ is $\left(x, \lfloor(\frac{ex}{d})^d\rfloor \right)$-class for every $x>d$, due to Sauer's lemma stated below.

\begin{theorem}[Sauer-Shelah Lemma \citep{S1972b,S1972a}]
\label{thm:sum}
Let $\calC \subseteq \{0,1\}^n$ be a concept class with $\vcd(\calC)=d$. Then for any $A \subseteq [n]$ such that $|A|>d$,
\[
\big|\{c|_{A}: c \in \calC\}\big| \le \sum_{k=0}^d \left(|A| \atop k\right) \le \left(\frac{e|A|}{d}\right)^d.
\]
\end{theorem}

Our main result is based on analysis of the largest possible best-case teaching dimension of all finite $(x,y)$-classes.

\begin{definition}
Define $f(x,y) = \sup_{\calC} \tdmin(\calC)$, where the supremum is taken over all finite $(x,y)$-class $\calC$.
\end{definition}

\citet{K1999} proved $f(2,3)=1$, and \citet{MSWY2015} proved $f(3,6)\leq3$.









\section{Main Results}
\label{sec:main}

In this section, we state and prove our main result. We show that for any finite concept class $\calC$, $\rtd(\calC)$ is quadratically upper bounded by $\vcd(\calC)$.

\begin{theorem}
\label{main}
For any concept class $\calC \subseteq \{0,1\}^n$ with $\vcd(\calC) = d$,
\[
\rtd(\calC) = O(d^2).
\]
\end{theorem}

We first give an informal description of the proof, in which we extend the techniques developed in \citep{K1999,MSWY2015,CCT2016}. The key idea of our approach is to analyze $f(x,y)$, the largest possible best-case teaching dimension for $(x,y)$-classes. The first step is to show a recursive formula for $f(x,y)$. The observation is that for a monotone increasing function $\phi(x)$ that grows substantially slower than $2^x$, we have
\[
f(x+1, \phi(x+1)) \le f(x, \phi(x)) + O(x).
\]
The recursive formula immediately leads to a quadratic upper bound $f(x,\phi(x)) \le O(x^2)$.

The second step is to select an appropriate function $\phi(\cdot)$. We choose $\phi(x) = \alpha^x$ for certain $\alpha \in (1,2)$. Next, we relate the VC dimension to $f(x,y)$. We show that for any finite concept class $\calC$ with $\vcd(\calC)=d$, $\calC$ must be an $(x, \alpha^x)$-class for some $x$ not much larger than $d$. In fact, it suffices when $x$ is a constant times of $d$. Combining the above arguments, we have shown that the best-case teaching dimension of $\calC$ is upper bounded by $O(d^2)$. Finally, a standard argument yields $\rtd(\calC) = O(d^2)$.

Now we give the formal proof of Theorem \ref{main}. The next lemma gives the recursive formula of $f(x,y)$.

\begin{lemma}
\label{lemma:recursion}
For any positive integer $x,y,z$ such that $y\leq 2^x-1$ and $z\leq 2y+1$, the following inequality holds:
\[
f(x+1,z)\leq f(x,y)+\left\lceil \frac{(y+1)(x-1)+1}{2y-z+2}\right\rceil.
\]
\end{lemma}

\begin{proof}
For convenience, let $k=\left\lceil \frac{(y+1)(x-1)+1}{2y-z+2}\right\rceil$. For any concept class $\calC \subseteq \{0,1\}^n$, we only need to show that if $\calC$ is an $(x+1,z)$-class, then
\[
\tdmin(\calC) \leq f(x,y)+k.
\]
If $n < k$, the theorem is trivial, because
\[
\tdmin(\calC) \leq n < k \leq f(x,y)+k.
\]
Assume $n \ge k$ in the rest of the proof. For any $Y \subseteq [n]$, $|Y| = k$, and any $b \in \{0,1\}^k$, define
\[
\calC^{Y,b}:=\{c\in \calC: c|_{Y}= b\}.
\]
Following the approach of \citep{K1999,MSWY2015,CCT2016}, we choose $Y^*, b^*$ among all possible $Y, b$ such that $\calC^{Y^*,b^*}$ is nonempty and has the \emph{smallest} size. Without loss of generality, we assume $b^* = \mathbf{0}$.

If $\calC^{Y^*,b^*}$ is an $(x,y)$-class, our proof is finished, because we can find a concept $c \in \calC^{Y^*,b^*}$ so that $c$ has a teaching set $T \subseteq [n] \backslash Y^*$ of size no more than $f(x,y)$ which distinguishes $c$ from all other concepts in $\calC^{Y^*,b^*}$. Then $T \cup Y^*$ is a teaching set that distinguishes $c$ from all other concepts in $\calC$. The fact that $|T \cup Y^*|\leq f(x,y)+k$ completes the proof.

Finally we show $\calC^{Y^*,b^*}$ is an $(x,y)$-class. Assume for the sake of contradiction that $\calC^{Y^*,b^*}$ is not an $(x,y)$-class. Then there exists $Z \subseteq [n]$ such that $|Z| \leq x $ and $|\{c|_{Z} : c\in \calC^{Y^*,b^*}\}|\geq y+1$. Note that $Z \backslash Y^*$ cannot be an empty set since $y+1>1$. Without loss of generality, we assume $Z \cap Y^* = \emptyset$; otherwise simply consider $Z \backslash Y^*$ instead of $Z$.

Now define
\[
\calC_Z^{Y^*,b^*} := \{c|_{Z} : c\in \calC^{Y^*,b^*}\},
\]
and for every $w \in Y^*$ define
\[
\calC_Z^{w,1} := \{c|_{Z} : c\in \calC, c|_{\{w\}}=1\}.
\]
Recall that $\calC$ is an $(x+1,z)$-class, $|Z| \le x$, and we assumed $b^*=\mathbf{0}$. Therefore, the projection of $\calC$ on the set $Z \cup \{w\}$ has no more than $z$ patterns. Thus
\[
\big| \calC_Z^{Y^*,b^*} \big| + \big| \calC_Z^{w,1} \big| \le z.
\]
Since $| \calC_Z^{Y^*,b^*} | \ge y+1$, we have $| \calC_Z^{w,1} | \le z-y-1$. Now, pick a subset $ \tilde{\calC}_Z^{Y^*,b^*} \subseteq \calC_Z^{Y^*,b^*} $ so that $ |\tilde{\calC}_Z^{Y^*,b^*}| = y+1$. We have for every $w \in Y^*$, $| \tilde{\calC}_Z^{Y^*,b^*} \backslash \calC_Z^{w,1} | \ge 2y-z+2$. Thus,
\begin{eqnarray*}
\sum\limits_{w\in Y^*}|\tilde{\calC}_Z^{Y^*,b^*} \backslash \calC_Z^{w,1}| & \geq &k(2y-z+2)\\
 & > & (y+1)(x-1)\\
 & = & |\tilde{\calC}_Z^{Y^*,b^*}|\cdot (x-1) \\
 & \geq & |\tilde{\calC}_Z^{Y^*,b^*}|\cdot(|Z|-1).
\end{eqnarray*}
It then follows from the Pigeonhole Principle that there exists $W \subseteq Y^*$ such that $|W|=|Z|$ and $\bigcap\limits_{w\in W}(\tilde{\calC}_Z^{Y^*,b^*} \backslash \calC_Z^{w,1})\neq\emptyset$. Pick any string $s \in \bigcap\limits_{w\in W}(\tilde{\calC}_Z^{Y^*,b^*} \backslash \calC_Z^{w,1})$, and consider the set $\calC^{(Y^* \backslash W) \cup Z, \mathbf{0}\circ s}$ defined as
\[
\calC^{(Y^* \backslash W) \cup Z, \mathbf{0}\circ s} := \{c \in \calC: c|_{(Y^* \backslash W)} = \mathbf{0}, ~c|_{Z} = s \}.
\]
It is clear that $\calC^{(Y^* \backslash W) \cup Z, \mathbf{0}\circ s}$ is a nonempty and proper subset of $\calC^{Y^*,b^*}$. This leads to a contradiction with the choice of $Y^*, b^*$.
\end{proof}

Using the recursive formula established in Lemma \ref{lemma:recursion}, we are able to give upper bound on the best-case teaching complexity for all $(x,y)$-classes.

\begin{lemma}
\label{lemma:bound_f}
For every $\alpha \in (1,2)$, and every positive integer $x$,
\[
f(x,\left\lfloor \alpha^x\right\rfloor)\leq \frac{(x-1)^2}{4-2\alpha}+\frac{3-2\alpha}{4-2\alpha}\cdot (x-1).
\]
\end{lemma}

\begin{proof}
Applying lemma \ref{lemma:recursion} by setting $y=\left\lfloor\alpha^x\right\rfloor$ and $z=\left\lfloor\alpha^{x+1}\right\rfloor$, we have
\begin{equation}
\label{equ}
f(x+1,\left\lfloor\alpha^{x+1}\right\rfloor)\leq f(x,\left\lfloor\alpha^x\right\rfloor)+\left\lceil \frac{(\left\lfloor\alpha^x\right\rfloor+1)(x-1)+1}{2\left\lfloor\alpha^x\right\rfloor-\left\lfloor\alpha^{x+1}\right\rfloor+2}\right\rceil.
\end{equation}
Since

\begin{equation*}
\left\lceil \frac{(\left\lfloor\alpha^x\right\rfloor+1)(x-1)+1}{2\left\lfloor\alpha^x\right\rfloor-\left\lfloor\alpha^{x+1}\right\rfloor+2}\right\rceil
\leq\frac{x-1}{2-\frac{\left\lfloor\alpha^{x+1}\right\rfloor}{\left\lfloor\alpha^{x}\right\rfloor+1}}+1
\leq\frac{x-1}{2-\alpha}+1
=\frac{x+1-\alpha}{2-\alpha},
\end{equation*}
Inequality (\ref{equ}) can be simplified to
$$f(x+1,\left\lfloor\alpha^{x+1}\right\rfloor)\leq f(x,\left\lfloor\alpha^x\right\rfloor)+\frac{x+1-\alpha}{2-\alpha}.$$
Observe that $f(1,1)=0$ and apply the above inequality recursively, we obtain
$$f(x,\left\lfloor \alpha^x\right\rfloor)\leq \frac{(x-1)^2}{4-2\alpha}+\frac{3-2\alpha}{4-2\alpha}\cdot (x-1).$$
\label{sec:bound}
\end{proof}

Next we show that for a concept class $\calC$ with $\vcd(\calC)=d$, $\calC$ must be an $(x, \left\lfloor \alpha^x\right\rfloor)$-class for $x$ not much larger than $d$.

\begin{lemma}
\label{lemma:VC2xy}
Given $\alpha \in (1,2)$, define
\[
\lambda^* := \inf \{ \lambda \ge 1: \lambda\ln\alpha-\ln\lambda-1 \geq 0\}.
\]
Then for any concept class $\calC \subseteq \{0,1\}^n$ with $\vcd(\calC)=d$, $\calC$ is an $(x,\left\lfloor \alpha^x\right\rfloor)$-class for every integer $x \ge \lambda^* d$.
\end{lemma}

\begin{proof}
By Sauer's lemma, we only need to verify
\begin{equation*}
\left(\frac{ex}{d}\right)^d\leq \alpha^x,
\end{equation*}
holds for all $x \ge \lambda^* d$. This follows from elementary calculus. We omit the details.
\end{proof}

Now we give the main conclusion.

\begin{theorem}
For any concept class $\calC \subseteq \{0,1\}^n$ with $\vcd(\calC)=d$,
\[
\rtd(\calC)\leq39.3752d^2-3.6330d.
\]
\end{theorem}

\begin{proof}
By Lemma \ref{lemma:bound_f} and Lemma \ref{lemma:VC2xy}, we have for any $\alpha \in (1,2)$ and any $x \ge \lambda^* d$, where $\lambda^*$ is defined in Lemma \ref{lemma:VC2xy}, the following holds
\[
\tdmin(\calC)\leq \frac{(x-1)^2}{4-2\alpha}+\frac{3-2\alpha}{4-2\alpha}\cdot (x-1).
\]
Observe that the VC dimension of a concept class does not increase after a concept is removed, we have
\begin{equation}
\label{eq:rtd_quad}
\rtd(\calC)\leq \frac{(x-1)^2}{4-2\alpha}+\frac{3-2\alpha}{4-2\alpha}\cdot (x-1).
\end{equation}

To optimize the coefficients in the quadratic bound, we choose $\lambda^*=4.71607,\alpha=(e\lambda^*)^{1/\lambda^*}\approx 1.71757,x=\lceil\lambda^* d\rceil$. Finally, observe that the RHS of (\ref{eq:rtd_quad}) is an increasing function of $x$ on the interval $[\lambda^*,+\infty)$ given our choice of the parameters, we conclude that
\[
\rtd(\calC)\leq \frac{(\lambda^* d)^2}{4-2\alpha}+\frac{3-2\alpha}{4-2\alpha}\cdot \lambda^* d\leq 39.3752d^2-3.6330d.
\]
\end{proof}

\section{Discussion and Conclusion}
\label{sec:discussion}

In the previous section we show that for finite concept class $\calC$, $\rtd(\calC) = O(\vcd(\calC)^2)$. In this section, we discuss our thoughts on the challenges in fully solving the open problem $\rtd(\calC) = O(\vcd(\calC))$.

The key technical result in our proof is the quadratic upper bound in Lemma \ref{lemma:bound_f} which, loosely speaking, is that for $\phi(x) < 2^x$
\begin{equation}
\label{bound_in_phi}
f(x,\phi(x)) = O(x^2),
\end{equation}
which is based on the recursive formula
\begin{equation}
\label{recursive_in_phi}
f(x+1,\phi(x+1)) \le f(x,\phi(x)) + O(x),
\end{equation}

In order to prove $\rtd(\calC) = O(\vcd(\calC))$ (if it is true), one needs to strengthen (\ref{bound_in_phi}) to
\[
f(x,\phi(x)) = O(x).
\]
If we still follow the recursive approach, we have to improve the recursive formula (\ref{recursive_in_phi}) to
\begin{equation}
\label{recursive_in_phi_improved}
f(x+1,\phi(x+1)) \le f(x,\phi(x)) + O(1).
\end{equation}
In our view, (\ref{recursive_in_phi_improved}) is qualitative different from (\ref{recursive_in_phi}); and this is the bottleneck of the current approach.

Another way to state the quadratic upper bound for $f(x,y)$ in Lemma \ref{lemma:bound_f} is $f(x,y) = O(\log^2 y)$ for $y<2^x$. (To see this, observe $f$ is non-increasing in $x$ and non-decreasing in $y$.) The conjecture $\rtd(\calC) = O(\vcd(\calC))$ is exactly equivalent to $f(x,y) = O(\log y)$ for all $y<2^x$. However, the only result we can show is that $f(x,y) = O(\log y)$ for $y$ not much larger than $x$.

We also consider the relation between RTD and VCD via the probabilistic method. If we fix $n$ and the size of the concept class as $N$ ($n,N$ sufficiently large), and draw $N$ concepts from $\{0,1\}^n$ uniformly at random to form $\calC$, it can be shown that with overwhelming probability $\rtd(\calC)$ is smaller than $\vcd(\calC)$. Although this does not prove any bound, it tells us that the cases $\rtd(\calC) \gg \vcd(\calC)$ are rare.

So far we focus on the upper bounds for RTD in terms of VCD, and discuss the challenges in proving $\rtd(\calC) = O(\vcd(\calC))$. What if RTD is not linearly upper bounded by VCD? How to prove it? There are attempts along this line. \citet{K1999} first showed there exist finite concept classes $\calC$ with $\rtd(\calC) = \frac{3}{2}\vcd(\calC)$. Warmuth discovered the smallest such class \citep{DFSZ2014}. \citet{CCT2016}, based on their insights and with the aid of SAT solvers, found finite concept classes $\calC$ with $\rtd(\calC) = \frac{5}{3}\vcd(\calC)$. However, to prove RTD is not linearly bounded by VCD, we need a sequence $\calC_1,\calC_2,\ldots$ ($\calC_i \subseteq \{0,1\}^{n_i}$) such that $\frac{\rtd(\calC_i)}{\vcd(\calC_i)}$ grows beyond any constant.

In order that $\frac{\rtd(\calC_i)}{\vcd(\calC_i)}$ grows unboundedly, $n_i$ and $|\calC_i|$ have to grow unboundedly as well. This means that as the instance space getting larger, there exist concept classes for which the ratio of RTD and VCD grows. However, currently there is no clue that larger $n_i$ and $|\calC_i|$ would result in larger ratio between RTD and VCD in a \emph{structural} way. The only known structural result is that for Cartesian product of two concept classes, the ratio does NOT grow. More concretely \citep{DFSZ2014},
\[
\rtd(\calC_1 \times \calC_2) \le \rtd(\calC_1) + \rtd(\calC_2),
\]
and
\[
\vcd(\calC_1 \times \calC_2) = \vcd(\calC_1) + \vcd(\calC_2).
\]
We believe any improvement along this line requires constructions more delicate in structure.

Our understanding of the quantitative relation between RTD and VCD is still preliminary. Even for the simple special case $\vcd =2$, we do not have a complete characterization: The best known upper bound for $\calC \subseteq \{0,1\}^n$ with $\vcd(\calC)=2$ is $\rtd(\calC) \le 6$ \citep{CCT2016}; and the worst-case lower bound is $\rtd(\calC)\ge3$ \citep{K1999,DFSZ2014}. The current knowledge of the four cases of $\vcd(\calC)=2$ (i.e., $(3,7), (3,6), (3,5), (3,4)$-classes) is not complete either: For $(3,7)$-classes, \citet{CCT2016} proved
\[
3 \le \max_{\calC \in (3,7)}\rtd(\calC) \le 6;
\]
For $(3,6)$-classes, \citet{MSWY2015} proved
\[
2 \le \max_{\calC \in (3,6)}\rtd(\calC) \le 3;
\]
Using a similar argument as in the proof of Lemma \ref{lemma:recursion} and optimizing the parameters with respect to the specific case of $(3,5)$-classes (choosing $x=2,y=3,z=5,k=1$), we can show 
\[
\max_{\calC \in (3,5)}\rtd(\calC) =2;
\]
And hence for $(3,4)$-classes $\max_{\calC \in (3,4)}\rtd(\calC) =2$.

The relationship between RTD and VCD is intriguing. Analyzing special cases and sub-classes with specific structure may provide insights for finally solving the open problem.

\clearpage
\newpage
\bibliographystyle{authordate1}
\bibliography{RTD_bound}

\begin{thebibliography}{}

\bibitem[\protect\citename{Angluin, }2004]{A2004}
Angluin, Dana. 2004.
\newblock Queries revisited.
\newblock {\em Theoretical Computer Science}, {\bf 313}(2), 175--194.

\bibitem[\protect\citename{Anthony {\em et~al.\ }\relax, }1995]{ABS1995}
Anthony, Martin, Brightwell, Graham, \& Shawe-Taylor, John. 1995.
\newblock On specifying {B}oolean functions by labelled examples.
\newblock {\em Discrete Applied Mathematics}, {\bf 61}(1), 1--25.

\bibitem[\protect\citename{Blumer {\em et~al.\ }\relax, }1989]{BEHW1989}
Blumer, Anselm, Ehrenfeucht, Andrzej, Haussler, David, \& Warmuth, Manfred~K.
  1989.
\newblock Learnability and the {V}apnik-{C}hervonenkis dimension.
\newblock {\em Journal of the ACM (JACM)}, {\bf 36}(4), 929--965.

\bibitem[\protect\citename{Chen {\em et~al.\ }\relax, }2016]{CCT2016}
Chen, Xi, Cheng, Yu, \& Tang, Bo. 2016.
\newblock On the Recursive Teaching Dimension of {VC} Classes.
\newblock {\em Pages  2164--2171 of:} {\em NIPS}.

\bibitem[\protect\citename{Dasgupta, }2005]{D2005}
Dasgupta, Sanjoy. 2005.
\newblock Coarse sample complexity bounds for active learning.
\newblock {\em Pages  235--242 of:} {\em NIPS},  vol. 18.

\bibitem[\protect\citename{David {\em et~al.\ }\relax, }2016]{DMY2016}
David, Ofir, Moran, Shay, \& Yehudayoff, Amir. 2016.
\newblock On statistical learning via the lens of compression.
\newblock {\em In:} {\em NIPS}.

\bibitem[\protect\citename{Doliwa {\em et~al.\ }\relax, }2014]{DFSZ2014}
Doliwa, Thorsten, Fan, Gaojian, Simon, Hans~Ulrich, \& Zilles, Sandra. 2014.
\newblock Recursive teaching dimension, {VC}-dimension and sample compression.
\newblock {\em Journal of Machine Learning Research}, {\bf 15}(1), 3107--3131.

\bibitem[\protect\citename{Goldman \& Kearns, }1995]{GK1995}
Goldman, Sally~A, \& Kearns, Michael~J. 1995.
\newblock On the complexity of teaching.
\newblock {\em Journal of Computer and System Sciences}, {\bf 50}(1), 20--31.

\bibitem[\protect\citename{Goldman \& Mathias, }1993]{GM1993}
Goldman, Sally~A, \& Mathias, H~David. 1993.
\newblock Teaching a smart learner.
\newblock {\em Pages  67--76 of:} {\em Proceedings of the sixth annual
  conference on Computational learning theory}.
\newblock ACM.

\bibitem[\protect\citename{Hanneke, }2007]{H2007}
Hanneke, Steve. 2007.
\newblock Teaching dimension and the complexity of active learning.
\newblock {\em Pages  66--81 of:} {\em International Conference on
  Computational Learning Theory}.
\newblock Springer.

\bibitem[\protect\citename{Heged{\H{u}}s, }1995]{H1995}
Heged{\H{u}}s, Tibor. 1995.
\newblock Generalized teaching dimensions and the query complexity of learning.
\newblock {\em Pages  108--117 of:} {\em Proceedings of the eighth annual
  conference on Computational learning theory}.
\newblock ACM.

\bibitem[\protect\citename{Kuhlmann, }1999]{K1999}
Kuhlmann, Christian. 1999.
\newblock On teaching and learning intersection-closed concept classes.
\newblock {\em Pages  168--182 of:} {\em European Conference on Computational
  Learning Theory}.
\newblock Springer.

\bibitem[\protect\citename{Littlestone \& Warmuth, }1986]{LW1986}
Littlestone, Nick, \& Warmuth, Manfred. 1986.
\newblock {\em Relating data compression and learnability}.
\newblock Tech. rept. Technical report, University of California, Santa Cruz.

\bibitem[\protect\citename{Moran {\em et~al.\ }\relax, }2015]{MSWY2015}
Moran, Shay, Shpilka, Amir, Wigderson, Avi, \& Yehudayoff, Amir. 2015.
\newblock Compressing and teaching for low {VC}-dimension.
\newblock {\em Pages  40--51 of:} {\em Foundations of Computer Science (FOCS),
  2015 IEEE 56th Annual Symposium on}.
\newblock IEEE.

\bibitem[\protect\citename{Sauer, }1972]{S1972b}
Sauer, Norbert. 1972.
\newblock On the density of families of sets.
\newblock {\em Journal of Combinatorial Theory, Series A}, {\bf 13}(1),
  145--147.

\bibitem[\protect\citename{Shelah, }1972]{S1972a}
Shelah, Saharon. 1972.
\newblock A combinatorial problem; stability and order for models and theories
  in infinitary languages.
\newblock {\em Pacific Journal of Mathematics}, {\bf 41}(1), 247--261.

\bibitem[\protect\citename{Shinohara \& Miyano, }1991]{SM1991}
Shinohara, Ayumi, \& Miyano, Satoru. 1991.
\newblock Teachability in computational learning.
\newblock {\em New Generation Computing}, {\bf 8}(4), 337--347.

\bibitem[\protect\citename{Simon \& Zilles, }2015]{SZ2015}
Simon, Hans~Ulrich, \& Zilles, Sandra. 2015.
\newblock Open Problem: Recursive Teaching Dimension Versus {VC} Dimension.
\newblock {\em Pages  1770--1772 of:} {\em COLT}.

\bibitem[\protect\citename{Vapnik \& Chervonenkis, }1971]{VC1971}
Vapnik, VN, \& Chervonenkis, A~Ya. 1971.
\newblock On the Uniform Convergence of Relative Frequencies of Events to Their
  Probabilities.
\newblock {\em Theory of Probability \& Its Applications}, {\bf 16}(2),
  264--280.

\bibitem[\protect\citename{Warmuth, }2003]{W2003}
Warmuth, Manfred~K. 2003.
\newblock Compressing to {VC} dimension many points.
\newblock {\em Pages  743--744 of:} {\em COLT},  vol. 3.
\newblock Springer.

\bibitem[\protect\citename{Zilles {\em et~al.\ }\relax, }2011]{ZLHZ2011}
Zilles, Sandra, Lange, Steffen, Holte, Robert, \& Zinkevich, Martin. 2011.
\newblock Models of cooperative teaching and learning.
\newblock {\em Journal of Machine Learning Research}, {\bf 12}(Feb), 349--384.

\end{thebibliography}

\end{document}